\documentclass{article}
\usepackage{hobo_nips}

\usepackage[utf8]{inputenc} %
\usepackage[T1]{fontenc}    %
\usepackage{hyperref}       %
\usepackage{url}            %
\usepackage{booktabs}       %
\usepackage{amsfonts}       %
\usepackage{nicefrac}       %
\usepackage{microtype}      %

\usepackage{amsmath}
\usepackage{amssymb}
\usepackage{amsthm}
\usepackage{algorithm}
\usepackage{algorithmic}
\usepackage{color}
\usepackage{graphicx}
\newtheorem{theorem}{Theorem}
\newtheorem{lemma}{Lemma}
\newtheorem{assumption}{Assumption}

\newtheorem{definition}{Definition}

\usepackage{wrapfig}

\usepackage[round,authoryear]{natbib} %
\title{\bf {\LARGE On Transformations in Stochastic Gradient MCMC}}

\author{{\bf Soma Yokoi$^{1,2}$,~ ~  Takuma Otsuka$^{3}$,~ ~  Issei Sato$^{1,2}$}\\
1~ The University of Tokyo\\
2~ RIKEN\\
3~ NTT Communication Science Laboratories, NTT Corporation
}
\date{}

\begin{document}

\maketitle

\begin{abstract}
Stochastic gradient Langevin dynamics (SGLD) is a computationally efficient sampler for Bayesian posterior inference given a large scale dataset.
Although SGLD is designed for unbounded random variables, many practical models incorporate variables with boundaries such as non-negative ones or those in a finite interval.
To bridge this gap, we consider mapping unbounded samples into the target interval.
This paper reveals that several mapping approaches commonly used in the literature produces erroneous samples from theoretical and empirical perspectives. 
We show that the change of random variable using an invertible Lipschitz mapping function overcomes the pitfall as well as attains the weak convergence.
Experiments demonstrate its efficacy for widely-used models with bounded latent variables including Bayesian non-negative matrix factorization and binary neural networks.
\end{abstract}

\section{Introduction} \label{sec:intro}
Sampling a random variable from a given target distribution is a key problem in Bayesian inference.
In this study, we discuss the problem of drawing samples from a target distribution on a bounded domain using the Langevin Monte Carlo (LMC) algorithm.
More precisely, let $\theta \sim \pi_\theta(\theta)$ be the target random variable in constrained state space $\mathbb{R}_c$ and $\varphi \sim \pi(\varphi)$ be a proxy random variable in $\mathbb{R}$. 
While we are interested in sampling from $\pi_\theta(\theta)$, LMC is unsuitable for directly handling such constrained random variables because its diffusion is prone to overstep the boundary. 
Thus, we discuss the following two-step LMC algorithm:
\begin{equation}\label{equ:target_algo}
    \varphi_{t+1} = \varphi_t + \epsilon \widehat{\nabla} \log \pi(\varphi_t) + \sqrt{2\epsilon} \eta_t, \ \ \ \
    \theta_{t+1} = f(\varphi_{t+1}),
\end{equation}
where $f$ is a transform function that maps the proxy to the target domain and $\widehat{\nabla}$ denotes an unbiased stochastic gradient operator.
This kind of algorithm is often employed when $\theta$ is difficult to directly sample.
For example, when $\theta$ is non-negative, the exponential function is adopted as mapping $f$.

The target $\pi_\theta(\theta)$ can be complex (e.g. neural networks) and dataset can become very large.
This forces us to comply with the following requirements.
First, $\widehat{\nabla} \log \pi(\varphi)$ must be designed such that the resultant distribution of $\theta$ should match $\pi_\theta (\theta)$ through the chosen mapping $f$.
Second, we have to avoid iterative evaluations of the whole dataset.
This means that 1) we use stochastic gradient with minibatch, and that 2) we omit the Metropolis-Hastings rejection step to avoid performance overhead as in previous studies \citep{Welling:2011aa} \citep{Sato:2014xy} \citep{Teh:2016aa}.
Both approximations introduce sampling errors as shown in Figure\,\ref{fig:basic_compare}.
Thus the sampling accuracy must be guaranteed by discretization analysis, instead of confirming the detailed balance of Markov chain.

The following three algorithms conforming to Eq.\,\eqref{equ:target_algo} are discussed in this paper.
\begin{itemize}
    \item Mirroring trick (Section\,\ref{sec:mirror}): heuristics employed in \cite{Patterson:2013kq}, simply matching the domain e.g. $f(x)=|x|$ for non-negative $x$ assuming $\widehat{\nabla} \log \pi(\varphi) = \widehat{\nabla} \log \pi_\theta(\theta)$.
    \item It\^o formula (Section\,\ref{sec:Ito}): transformation $f$ in stochastic differential equation~(SDE), obtaining $\widehat{\nabla} \log \pi(\varphi)$ by It\^o formula, the chain rule in stochastic calculus.
    \item Change of random variable (CoRV) (Section\,\ref{sec:CoRV}): transformation $f$ in random variable,
    obtaining $\widehat{\nabla} \log \pi(\varphi)$ from that of $\theta$ by Jacobian.
\end{itemize}
It turns out that there are theoretical and empirical problems with straightforward use of these methods.
Mirroring trick suffers from inaccurate sampling near a boundary and no theoretical guarantee with stochastic gradient.
It\^o formula almost surely diverges near a boundary and causes a stepsize issue.
Only CoRV with Lipschitz $f$ gives good results both theoretically and empirically.

\citet{Brosse:2017aa} developed another line of research for an LMC algorithm for a random variable on a convex body.
They employed proximal MCMC~\citep{Pereyra:2016:PMC:2967713.2967757} \citep{Durmus:2018ab} with the Moreau-Yosida envelope, which find a well-behaved regularization of the target density on a convex body so that it preserves convexity and Lipschitzness.
The sampling distribution is, nevertheless, an unbounded approximation of the target distribution and it still draws samples from outside the domain.
The limitation of log-concavity and the computing cost of proximal operator at each sample prevent its application to large datasets as well as complex models such as neural networks.

\paragraph{Contribution.}
The contribution of this paper on sampling bounded random variables using a stochastic gradient-based sampler is twofold. 
\begin{itemize}
  \item We reveal that common practices used in the literature (the mirroring trick and the application of It\^o formula) have pitfalls from empirical and theoretical points of view.
  \item We guarantee that the CoRV approach has the stationary distribution (Theorem\,\ref{thm:proxy_stationary}) and weak convergence (Theorem\,\ref{thm:proxy_weak}) with a mild condition on the transform function (Assumption\,\ref{asm:transform_f}).
\end{itemize}

\begin{figure*}[t]\begin{center}
  \includegraphics[width=0.32\linewidth]{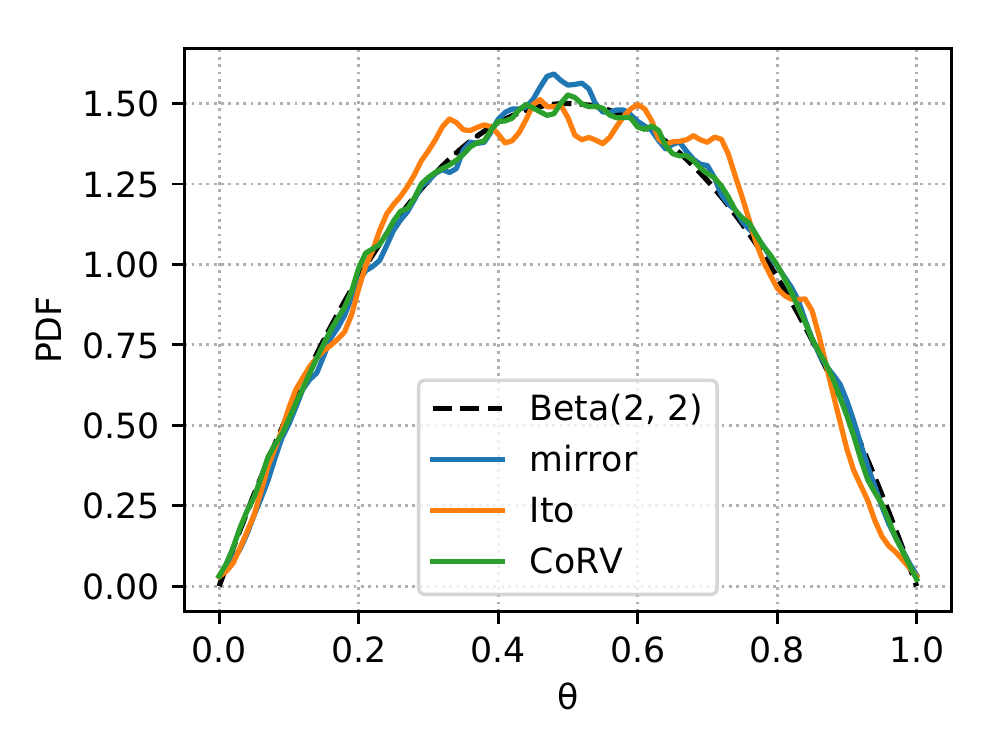}
  \includegraphics[width=0.32\linewidth]{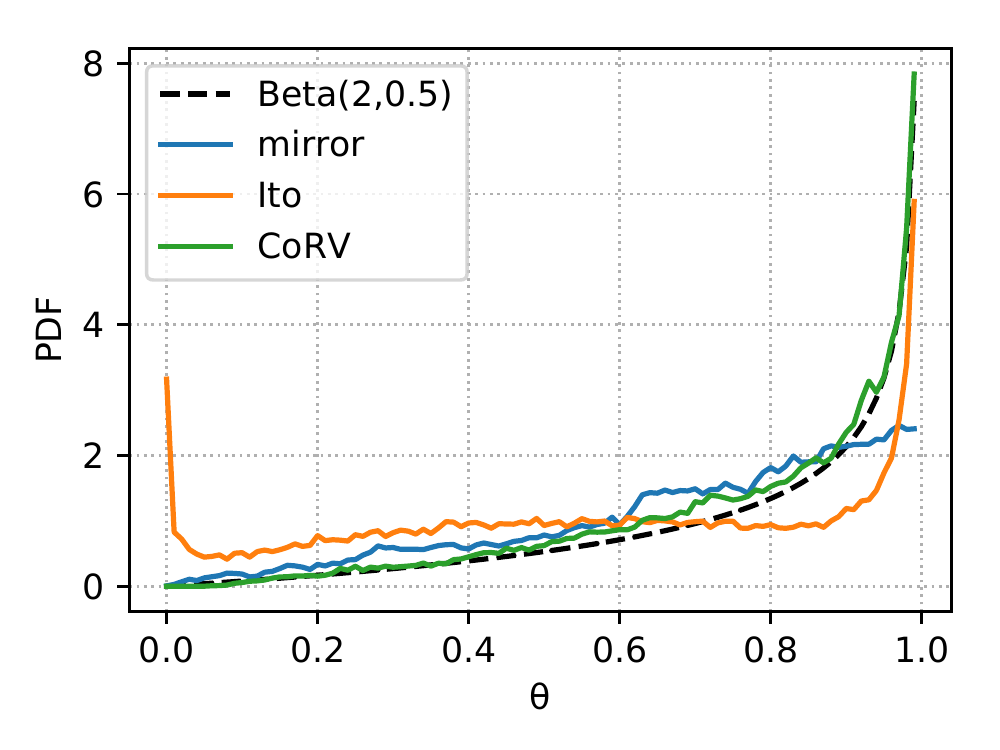}
  \includegraphics[width=0.32\linewidth]{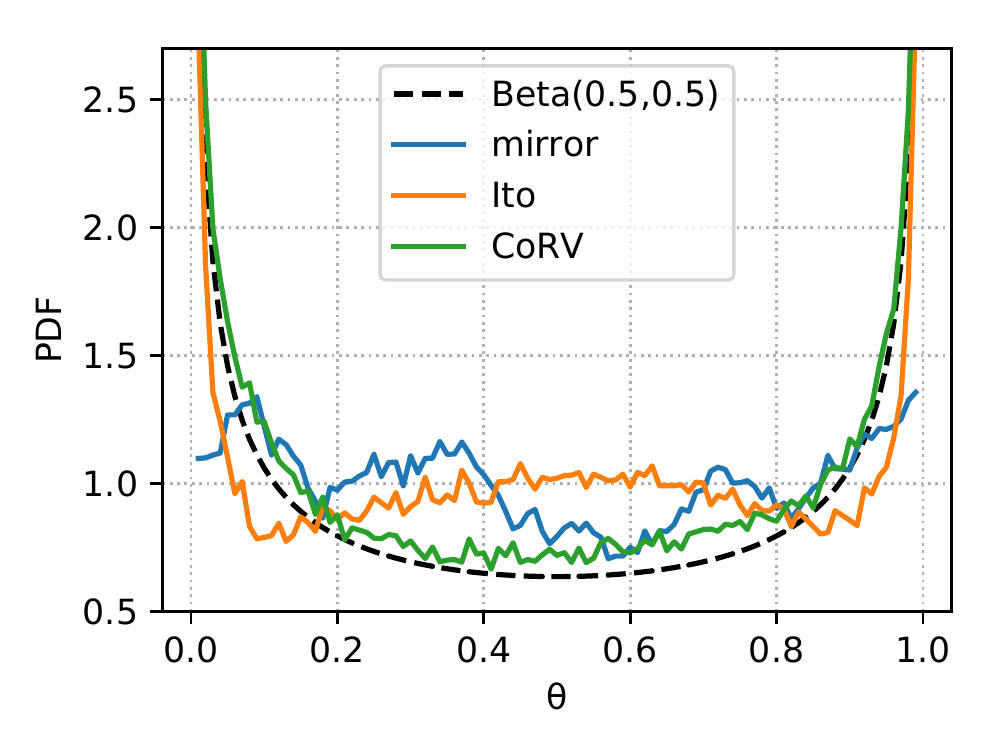}
  \includegraphics[width=0.32\linewidth]{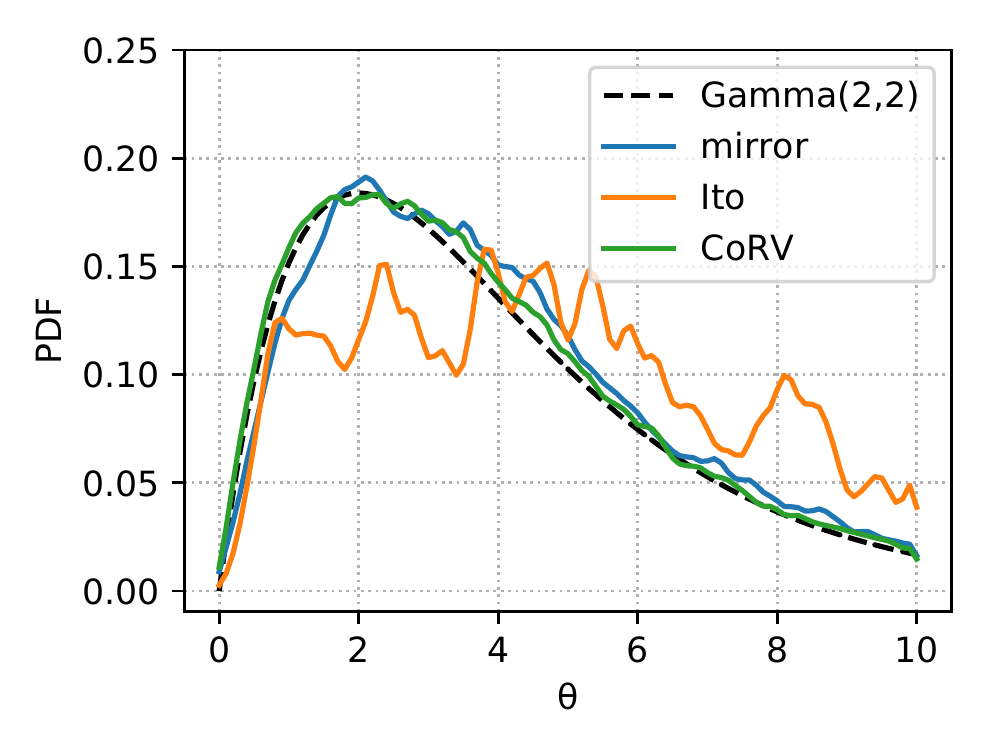}
  \includegraphics[width=0.32\linewidth]{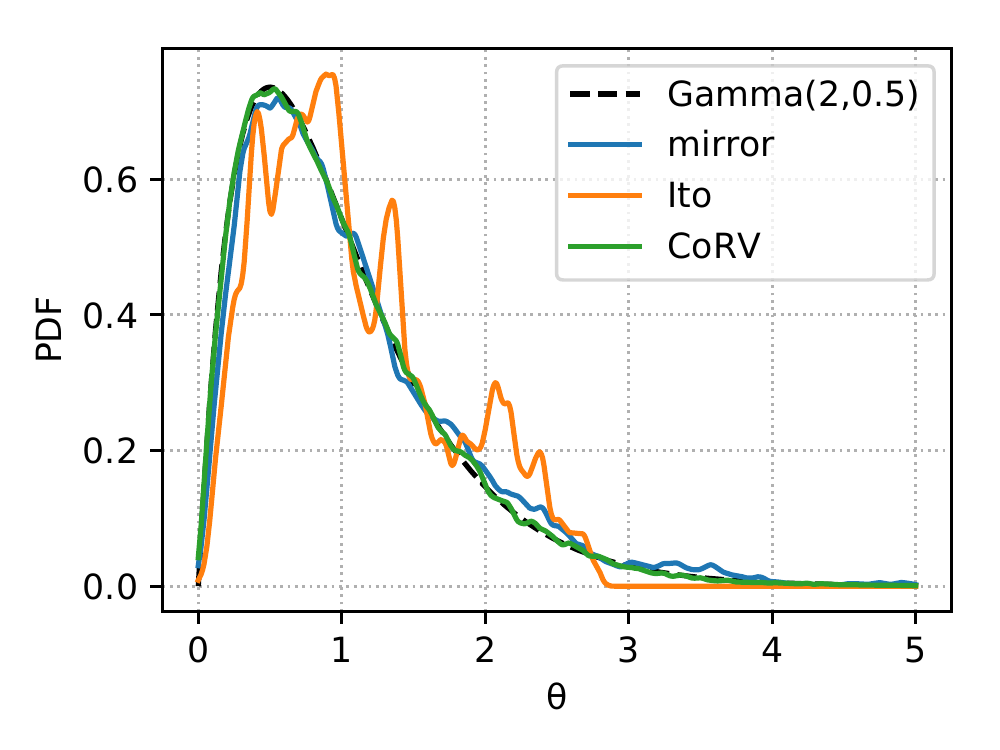}
  \includegraphics[width=0.32\linewidth]{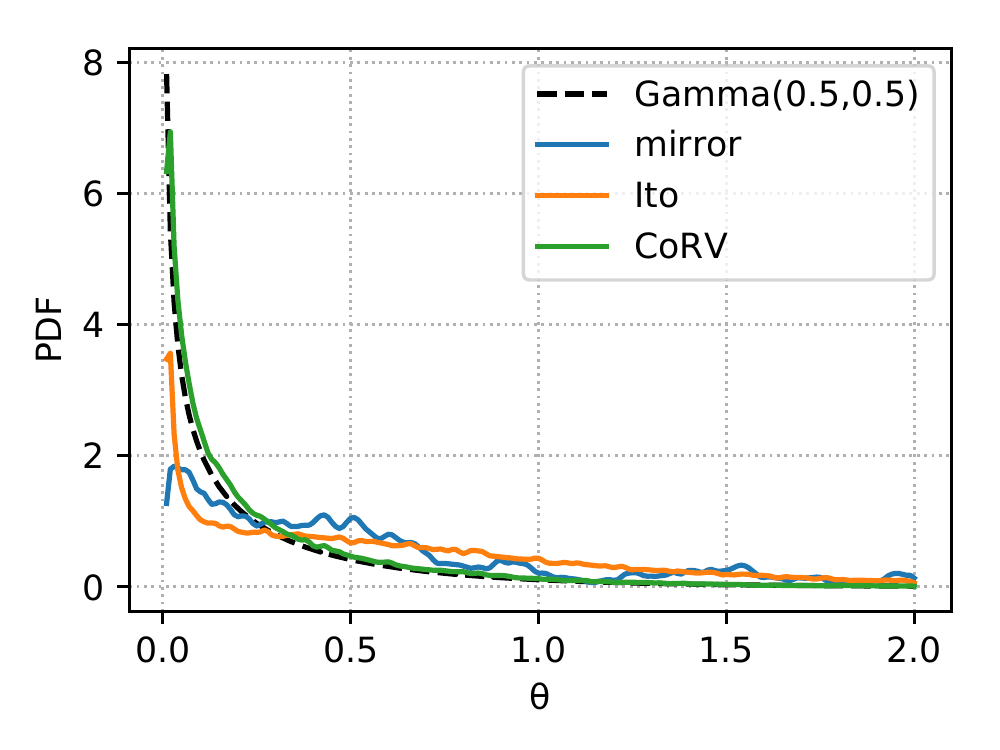}
  \includegraphics[width=0.32\linewidth]{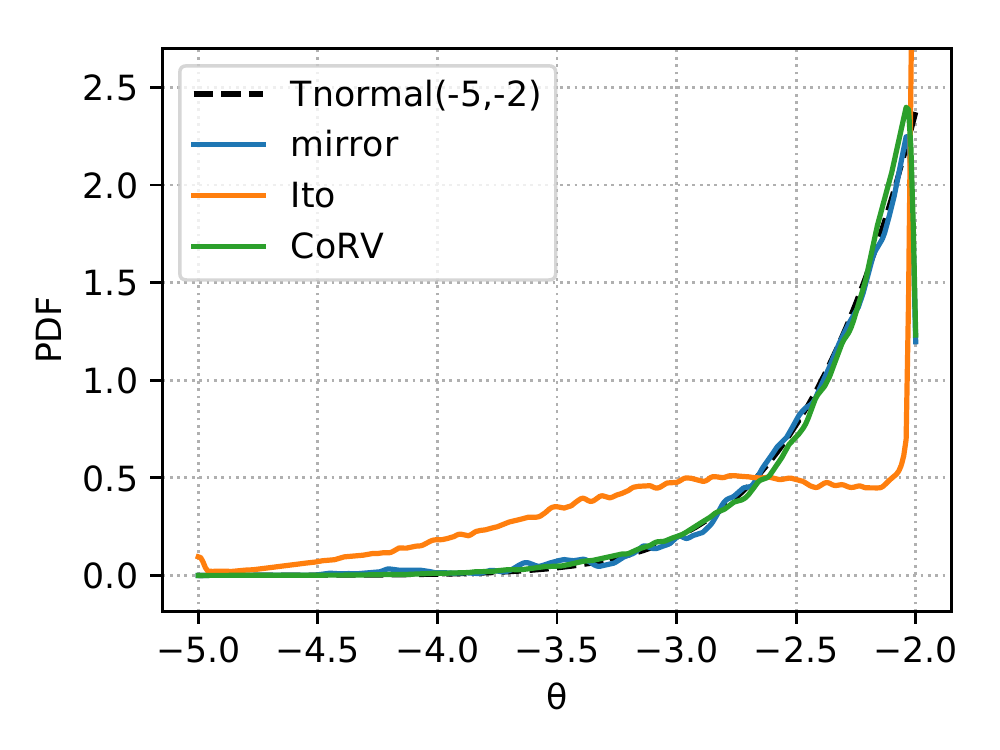}
  \includegraphics[width=0.32\linewidth]{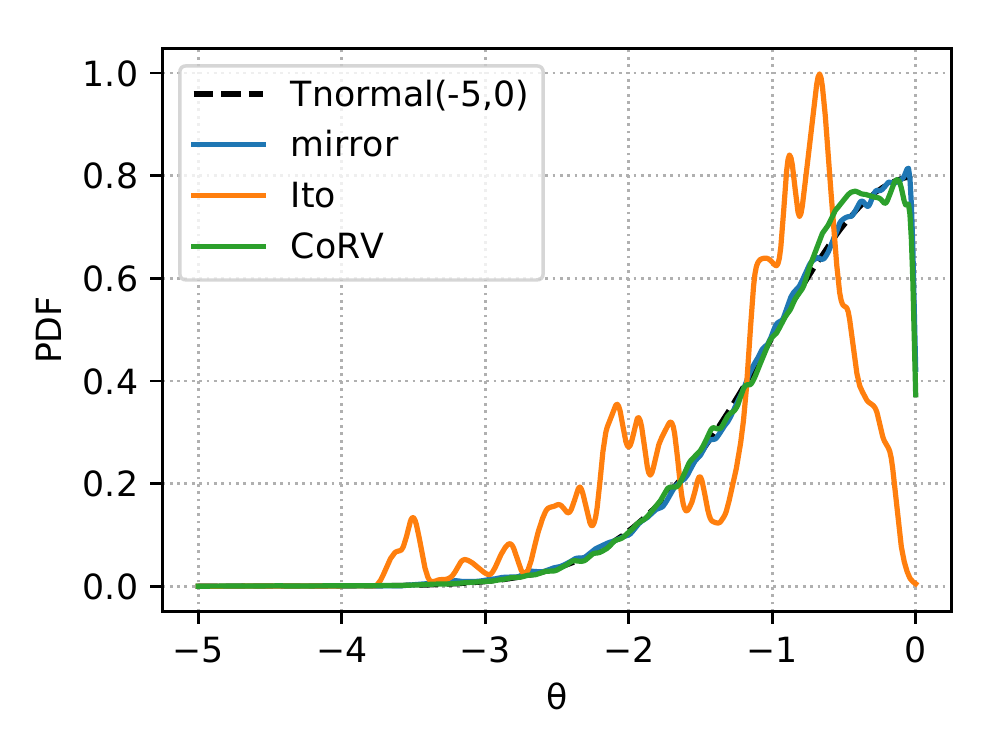}
  \includegraphics[width=0.32\linewidth]{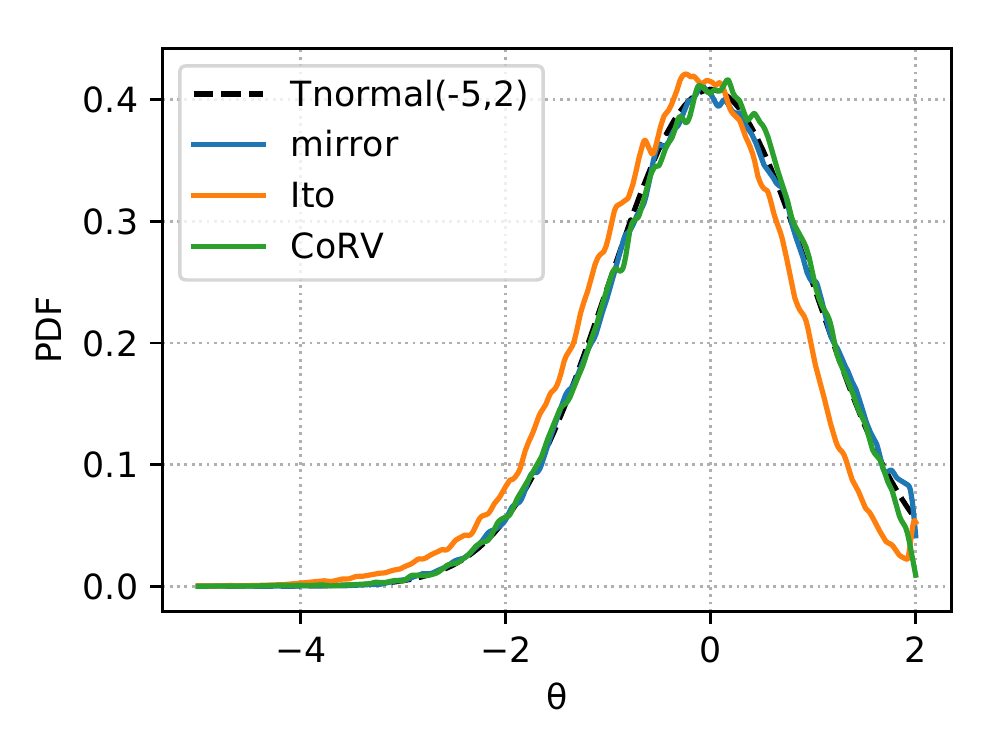}
  \caption{
    Sampling results from the beta, gamma, and truncated standard normal distributions with $n=100,000$ samples for each method.
    The mirroring trick~(\texttt{mirror}) often fails at the distributions with high density on their boundaries.
    The It\^o formula~(\texttt{Ito}) suffers from instability near boundaries as well as slow mixing due to a small stepsize.
    The change-of-random-variable formulation~(\texttt{CoRV}) works appropriately for these distributions.
    The stochastic gradients were emulated by adding Gaussian noise to the exact gradients.
    The stepsize was chosen by the tree-structured Parzen estimator\,(TPE)\,\citep{NIPS2011_4443} to maximize similarities between the true density functions and the histograms.
  }
  \label{fig:basic_compare}
\end{center}\end{figure*}

\section{Review: stochastic gradient Langevin dynamics} \label{sec:sgld}
This section reviews the SGLD algorithm in unconstrained state space.
Our notation uses a one-dimensional parameter for simplicity.
An extension to multi-dimensional cases is straightforward.

Consider a target potential $U_\theta(\theta)$ such that its Gibbs distribution is the target distribution $\pi_\theta(\theta) \propto \exp (-U_\theta(\theta))$.
We discuss an It\^o process described by the following SDE
\begin{equation} \label{equ:Ito_process}
  d\theta(t) = - U'_\theta(\theta)dt + \sqrt{2} dW(t),
\end{equation}
where $U'_\theta(\theta) = \frac{d}{d\theta} U_\theta(\theta)$ and $W(t)$ denotes the Wiener process.
By applying the first order Euler-Maruyama discretization and stochastic approximation, the SGLD algorithm is derived
\begin{equation} \label{equ:SGLD}
  \theta_{t+1} = \theta_t - \epsilon_t \widehat{U}'_\theta(\theta_t) + \sqrt{2 \epsilon_t} \eta_t, \ \ \ \eta_t \sim \mathcal{N}(0,1),
\end{equation}
where $\mathcal{N}(0,1)$ is the standard Gaussian distribution and $\epsilon_t>0$ is stepsize.
SGLD also enjoys computational gain by omitting a Metropolis-Hastings rejection step which ordinary MCMC methods usually runs to ensure detailed balance. 

Due to the approximation of gradient and exclusion of rejection step, SGLD may not necessarily satisfy the detailed balance of the Markov chain. 
Instead, the weak convergence with regard to SDE\,\eqref{equ:Ito_process} have been discussed in the literature \citep{Sato:2014xy} \citep{Teh:2016aa}.
Let stochastic gradient satisfy the following assumtion.
\begin{assumption}[gradient error] \label{asm:delta_SGLD}
  The stochastic gradient $\widehat{U}'_\theta(\theta)$ is written by using the accurate gradient $U'_\theta(\theta)$ and the error $\delta$ as
  \begin{equation} \label{equ:stoc_appr}
    \widehat{U}'_\theta(\theta) = U'_\theta(\theta) + \delta,
  \end{equation}
  where $\delta$ is white noise or the Wiener process of zero mean and finite variance satisfying
    \begin{equation}
      \mathbb{E}_{S}[\delta] = 0, \ \ \ \ \ \mathbb{E}_{S}[|\delta|^l] < \infty,
    \end{equation}
    for some integer $l \geq 2$.
  $\mathbb{E}_{S}$ denotes the expectation over sampling set $S$.
\end{assumption}

Then the following theorem holds for the sample sequence $\{\theta_t\}_{t=1}^T$.
In short, the weak convergence states that the discretization error of SGLD becomes zero in expectation for any fixed time where the time increment approaches zero.
\begin{definition}[weak convergence~\citep{Iacus:2008aa}]
  Let $Y_\epsilon$ be a time-discretized approximation of a continuous-time process $Y$ and $\epsilon_0$ be the maximum time increment of the discretization.
    $Y_\zeta$ is said to converge weakly to $Y$ if for any fixed time $T$ and any continuous differentiable and polynomial growth function $h$ and constant $\epsilon_0>0$, it holds true that
    \begin{equation}
      \lim_{\epsilon \to 0} \left|\mathbb{E}[h(Y_\epsilon(T))] - \mathbb{E}[h(Y(T))]\right| = 0, \ \ \ \ \forall \epsilon < \epsilon_0.
    \end{equation}
\end{definition}

\section{Mirroring trick} \label{sec:mirror}
Although many studies have been carried out for LMC and SGLD defined on real space $\mathbb{R}$, theoretical analysis in the finite interval $\mathbb{R}_c$ remains unsolved.
The difficulty comes from that the LMC algorithm is an Euler-Maruyama discretization of an It\^o process whose equilibrium is a target distribution on $\mathbb{R}$.
This is problematic in multiple applications where we handle latent random variables in a bounded domain, such as latent Dirichlet allocation~\citep{Blei:2003ly} where $\theta$ lies in a probability simplex, non-negative matrix factorization~\citep{Cemgil:2009:BIN:1592511.1592515} with all elements of $\theta$ being non-negative, and binary neural networks~\citep{Courbariaux:2015aa} \citep{NIPS2016_6573} with $\theta \in (-1, 1)$.

The mirroring trick is one of the straightforward heuristics to cope with this problem.
This trick sends back outgoing samples at the domain boundaries so as not to overstep the constraint.
\citet{Patterson:2013kq} employed it to sample from a Gamma distribution defined on $\mathbb{R}_+$, simply taking the absolute value of the generated sample.
There is no convergence guarantee for this trick, because it assumes that $\widehat{\nabla} \log \pi(\varphi) = \widehat{\nabla} \log \pi_\theta(\theta)$ and transformation $f$ does not change the equilibrium.
The heuristics is partially justified by \citet{Bubeck:2015aa} and \citet{Bubeck:2018aa}.
They extended the LMC algorithm with accurate gradients to an SDE with a reflecting boundary condition.
Their stochastic process defined on a convex body, called reflected Brownian motion, is discretized into an LMC algorithm accompanied by the mirroring trick.
This interpretation helps its theoretical investigation.
However, \citet{Bubeck:2018aa} also stated that the extension of their result to SGLD with stochastic gradients is an open problem for future work.

Our preliminary experiments show that the mirroring trick empirically suffers from inaccurate sampling near the boundaries.
Figure\,\ref{fig:basic_compare} (see \texttt{mirror}) indicates that the mirroring trick fails to capture the distribution especially when the density is sparse, or concentrated at boundaries.
This implies that the sampling may be inaccurate when the model uses a sparse prior that is often employed to avoid overfitting.
This disadvantage forces us to set a very small stepsize for accurate sampling near the boundary, which results in a large performance degradation in experiments in Section\,\ref{sec:experiment}.

\section{It\^o formula} \label{sec:Ito}
Here we consider the following two-step modification: first, we use the It\^o formula to construct the SDE in the unconstrained domain with the corresponding transform function.
Then, the SDE is discretized to obtain the desired algorithm.
While this derivation is straightforward and theoretically appreciated, we later show that this transformation inherits an instability near the boundary. 

We begin by transforming the following It\^o process of $\theta(t)$,
\begin{equation}
  d\theta(t) = a(t,\theta)dt + b(t,\theta) dW(t).
\end{equation}
Let $g:\mathbb{R}_c \to \mathbb{R}$ be a smooth invertible function from a bounded target variable $\theta \in \mathbb{R}_c$ to an unbounded proxy variable $\varphi \in \mathbb{R}$.
$\mathbb{R}_c$ is constrained state space, e.g. finite or semi-infinite interval for $\mathbb{R}$.
We consider a new stochastic process $\varphi(t)$ defined by
\begin{equation}
  \varphi(t) = g(\theta(t)).
\end{equation}
From the It\^o formula (Theorem\,\ref{theorem:ito_formula} in Appendix\,\ref{app:ito_formula}), $\varphi(t)$ is also an It\^o process of
\begin{equation}
    d\varphi(t) = \left\{a(t,\theta) g'(\theta(t)) + \frac{b^2}{2} g''(\theta(t)) \right\} dt + b(t,\theta) g'(\theta(t)) dW(t).
\end{equation}
Letting $a(\theta) = - U'_\theta(\theta)$ and $b = \sqrt{2}$, discretizing the process results in the following LMC
\begin{equation} \label{equ:naive_sampling}
  \varphi_{t+1} =  \varphi_t + \epsilon\left(- g'(\theta_t)U'_\theta(\theta_t) + g''(\theta_t)\right) + \sqrt{2\epsilon} g'(\theta_t) \eta.
\end{equation}
While a general connection between SDE and LMC is discussed by \citet{Ma:2015:CRS:2969442.2969566}, this algorithm is distinct in that the transform step $\theta = g^{-1}(\varphi)$ is employed to keep samples in the target domain.

Unfortunately, Eq.\,\eqref{equ:naive_sampling} is likely to draw inaccurate samples.
Figure\,\ref{fig:basic_compare} demonstrates that this method (labeled as \texttt{Ito}) fails to track the target density.
We attribute this phenomenon to the intrinsic instability around the boundary regardless of the target potential and the transform function.

To theoretically discuss this instability, we first assume the following class of transform functions.
\begin{assumption}[transform function] \label{asm:transform_f}
    Let $f$ be a Lipschitz and monotonically increasing function.
    Namely, for any $\varphi \in \mathbb{R}$, there exists constant $L>0$ such that
    \begin{equation}
        0 \leq f'(\varphi) \leq L.
    \end{equation}
    The boundary value of target domain denoted by $\partial S$ corresponds to the infinity in the proxy space: $\lim_{\varphi\to\infty} f(\varphi) = \partial S$, and $\lim_{\varphi\to\infty} f'(\varphi)$ exists.
\end{assumption}
All functions in Table~\ref{tab:transform} satisfy this assumption except the exponential.
Depending on the constraints in the target domain, $f$ may be a decreasing or upper- and lower-bounded function.
Though our discussion also applies to these cases in the same way, we continue with Assumption\,\ref{asm:transform_f} for simplicity.

Then the instability of the algorithm is shown in the following theorem.
\begin{theorem}[instability of the It\^o transformation] \label{thm:naive_unstable}
    Let $f=g^{-1}: \mathbb{R}\to\mathbb{R}_c$ satisfy Assumption~\ref{asm:transform_f}.
    Then for any $\epsilon>0$, and $U'_\theta(\theta)$, and any $\theta \in S$ approaching $\partial S$ from the inside, the single-step difference of the It\^o transformation method diverges almost surely:
    \begin{equation}
        \lim_{\theta\to\partial S} |\varphi_{t+1} - \varphi_t| = \infty
    \end{equation}
\end{theorem}
Please refer to Appendix\,\ref{app:proof} for all the proofs in this paper.

It suggests that the stepsize must be small enough to cope with this instability, but it would make the sampling substantially slow to mix.

\section{Change of random variable} \label{sec:CoRV}
We thus introduce another formulation to employ a transformation step in LMC.
The derivation methodology is the opposite of the It\^o transformation; we begin with a discretized algorithm and then consider the corresponding continuous-time SDE.
This SDE representation is used to derive Theorem\,\ref{thm:proxy_weak}, which guarantees the sampling accuracy of the method without a rejection step.
In addition, this algorithm overcomes the instability issue by Theorem\,\ref{thm:CoRV_stable} unlike the former It\^o method.

Let function $f: \mathbb{R} \to \mathbb{R}_c$ be a twice differentiable monotonic function from an unbounded proxy variable $\varphi \in \mathbb{R}$ to a bounded target variable $\theta \in \mathbb{R}_c$
\begin{equation} \label{equ:transform_back}
  \theta = f(\varphi),
\end{equation}
then the target density $\pi_\theta(\theta)$ and the proxy density $\pi(\varphi)$ are known to have the following relation,
\begin{equation}\label{equ:transform_dist}
  \pi(\varphi) = \pi_\theta(\theta) \left|f'(\varphi)\right|.
\end{equation}
For the proxy potential $U(\varphi) \propto - \log \pi(\varphi)$, proxy $U'(\varphi)$ is represented by given target $U'_\theta(\theta)$:
\begin{equation} \label{equ:convert_potential}
    U'(\varphi) = f'(\varphi) U'_\theta(\theta) - \frac{f''(\varphi)}{f'(\varphi)}.
\end{equation}
One can enjoy the computational gain using the stochastic gradient $\widehat{U}'_\theta$, and construct the SGLD algorithm for the proxy variable:
\begin{equation} \label{equ:propose_update}
  \varphi_{t+1} = \varphi_t - \epsilon_t \left( f'(\varphi_t) \widehat{U}'_\theta(\theta_t) - \frac{f''(\varphi_t)}{f'(\varphi_t)} \right) + \sqrt{2\epsilon} \eta_t.
\end{equation}
We call this algorithm change-of-random-variable (CoRV) SGLD.
CoRV SGLD forms a generalized class of samplers that contains the ordinary SGLD.
Indeed, we recover SGLD by using the identity function as the transform $f(\varphi) = \varphi$.
CoRV SGLD satisfies the following advantages.
\begin{itemize}
    \item The algorithm is computationally efficient.
    Equation\,\eqref{equ:propose_update} requires to iterate over minibatch.%
    \item The samples are always in the target constrained space $\mathbb{R}_c$.
    Equation\,\eqref{equ:propose_update} generates a proxy sample $\varphi_t \in \mathbb{R}$ and then Eq.\,\eqref{equ:transform_back} transforms it into a target sample $\theta_t \in \mathbb{R}_c$.
    \item Any transform functions $f$ can be employed in Eq.\,\eqref{equ:propose_update} if it is twice differentiable monotonic and $\frac{f''(\varphi)}{f'(\varphi)}$ exists.
    Many common functions satisfy this condition, such as exponential, sigmoid, and softmax functions.
\end{itemize}

\subsection{Stability}
The following theorem explains the stability of CoRV SGLD by showing that the transformation does not cause an abrupt movement in the dynamics.
\begin{theorem}[stability of CoRV] \label{thm:CoRV_stable}
    Let transform function $f$ satisfy Assumption~\ref{asm:transform_f}.
    Then for a gradient error $\delta_\varphi$ and for any $\theta \in S$ approaching $\partial S$ from the inside, we have:
    \begin{equation}
        \lim_{\theta \to \partial S} \delta_\varphi = 0.
    \end{equation}
\end{theorem}

\subsection{Stationary distribution}
We consider the following SDE of proxy variable $\varphi$ as the continuous counterpart of Eq.\,\eqref{equ:propose_update}
\begin{equation} \label{equ:proxy_process}
    d\varphi(t) = - \widehat{U}'(\varphi(t)) dt + \sqrt{2}dW(t),
\end{equation}
so as to apply the tools of stochastic analysis.
We confirm the existence and uniqueness of the weak solution and obtain its equilibrium.

Unlike the unconstrained case, a constrained target distribution $\pi_\theta(\theta)$ often has nonzero density at a domain boundary.
The following lemma is required so that the unnormalized proxy distribution $\int_{\varphi} \exp(-U(\varphi))d\varphi$ does not diverge.
\begin{lemma}[proxy potential] \label{lem:proxy_potential}
    Let $f$ satisfy Assumption~\ref{asm:transform_f}, and a target pdf $\pi_\theta(\theta)$ have a finite limit as $\theta$ goes to the boundary.
    Then for any $U(\varphi)$, we have:
    \begin{equation}
        \lim_{\varphi\to \infty} U(\varphi) = \infty.
    \end{equation}
\end{lemma}

Lemma~\ref{lem:proxy_potential} is enough for some cases (e.g. truncated normal).
However, in order to show the same proposition for distributions that has infinite density at a boundary (e.g. beta and gamma), we need the following additional assumption.
\begin{assumption}
    For $\pi_\theta(\theta)$ of interest, $f$ satisfies
    \begin{equation}
        \lim_{\varphi\to \infty} \pi_\theta(f(\varphi)) |f'(\varphi)| = 0.
    \end{equation}
\end{assumption}

Under the existence and uniqueness of solution (see Appendix\,\ref{app:proxyprocess_exist_unique}), we derive the stationary distribution of Eq.\,\eqref{equ:proxy_process} as follows.
\begin{theorem}[stationary distribution] \label{thm:proxy_stationary}
    Let transform function $f$ satisfy Assumption~\ref{asm:transform_f}.
    For transition probability density functions $p(\varphi,t)$ and $p(\theta,t)$ of the variables at time $t$, we have:
    \begin{equation}
        \lim_{t\to\infty} p(\varphi,t) = \pi(f(\varphi)) \left|f'(\varphi)\right|
        \quad \mathit{and} \quad
        \lim_{t\to\infty} p(\theta,t) = \pi_\theta(\theta).
    \end{equation}
\end{theorem}

\subsection{Weak convergence}
We also check Eq.\,\eqref{equ:propose_update} does not break the unique weak solution of Eq.\,\eqref{equ:proxy_process} by confirming that the discretization error is bounded.
From Lemmas~\ref{lem:proxy_potential} and \ref{lem:proxy_gradient_error}, the weak convergence is derived.
\begin{theorem}[weak convergence] \label{thm:proxy_weak}
    Let transform function $f$ satisfy Assumption~\ref{asm:transform_f}.
    For any test functions $h$ and $h_\theta$ those are continuous differentiable and polynomial growth, we have:
    \begin{equation}\label{equ:thm_weak_exp}
        \left|\mathbb{E}[h(\widetilde{\varphi}(T))] - \mathbb{E}[h(\varphi(T))]\right| = \mathcal{O}(\epsilon_0)
        \quad \mathit{and} \quad
        \left|\mathbb{E}[h_\theta(\widetilde{\theta}(T))] - \mathbb{E}[h_\theta(\theta(T))]\right| = \mathcal{O}(\epsilon_0),
    \end{equation}
    where $\varphi(T)$ and $\theta(T)$ denote the random variables at fixed time $T$, $\widetilde{\varphi}(T)$ and $\widetilde{\theta}(T)$ denote discretized samples at fixed time $T$ by CoRV SGLD, and $\epsilon_0>0$ is the initial stepsize.
\end{theorem}

\paragraph{Empirical result.}
We empirically confirm Theorem\,\ref{thm:proxy_weak} using basic distributions.
The expectation of continuous process $h_\theta (\theta(T))$ is substituted with its true expectation and the identity function $h_\theta(\theta) = \theta$ was selected. 
Specifically, we set $\mathbb{E}[h_\theta(\theta(T))] = 0.25$ for the gamma distribution with its shape and scale being $0.5$.
Figure\,\ref{fig:weakerror_compare} shows the numerical errors corresponding to Eq.\,\eqref{equ:thm_weak_exp} for three sampling methods. 
We can see that the error of \texttt{CoRV} almost linearly scales with stepsize $\epsilon_0$, as suggested by Theorem\,\ref{thm:proxy_weak}. 
The errors of \texttt{mirror} and \texttt{Ito} are significantly greater than \texttt{CoRV}.
The smaller stepsizes do not improve \texttt{Ito}, implying the difficulty for practical application.

\begin{figure*}[t]\begin{center}
  \includegraphics[width=0.32\linewidth]{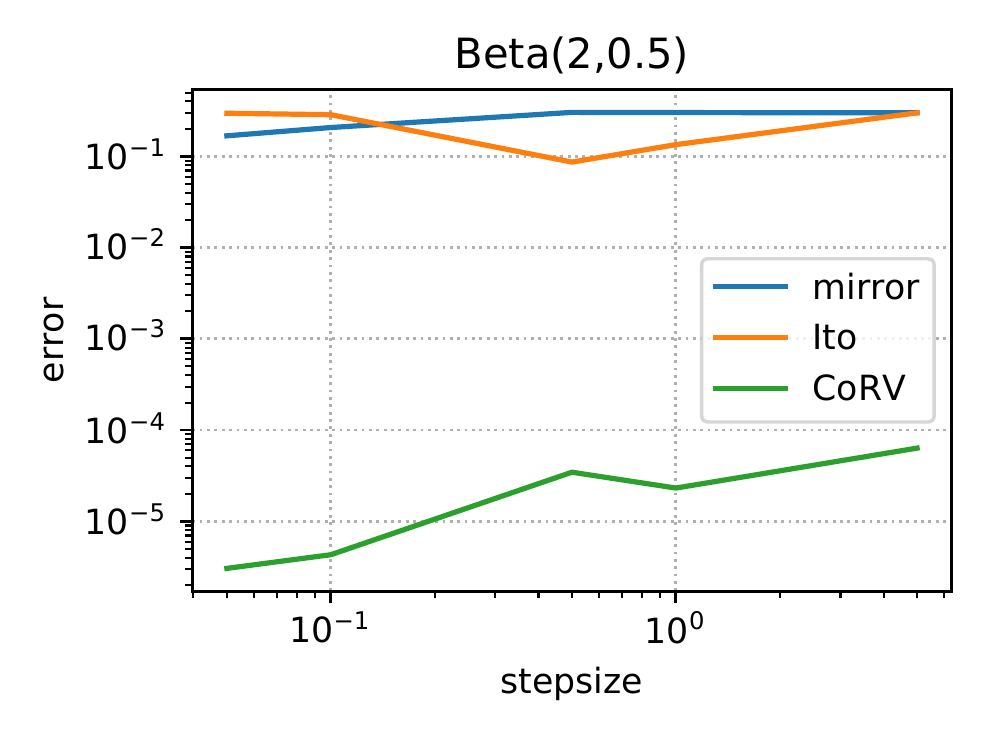}
  \includegraphics[width=0.32\linewidth]{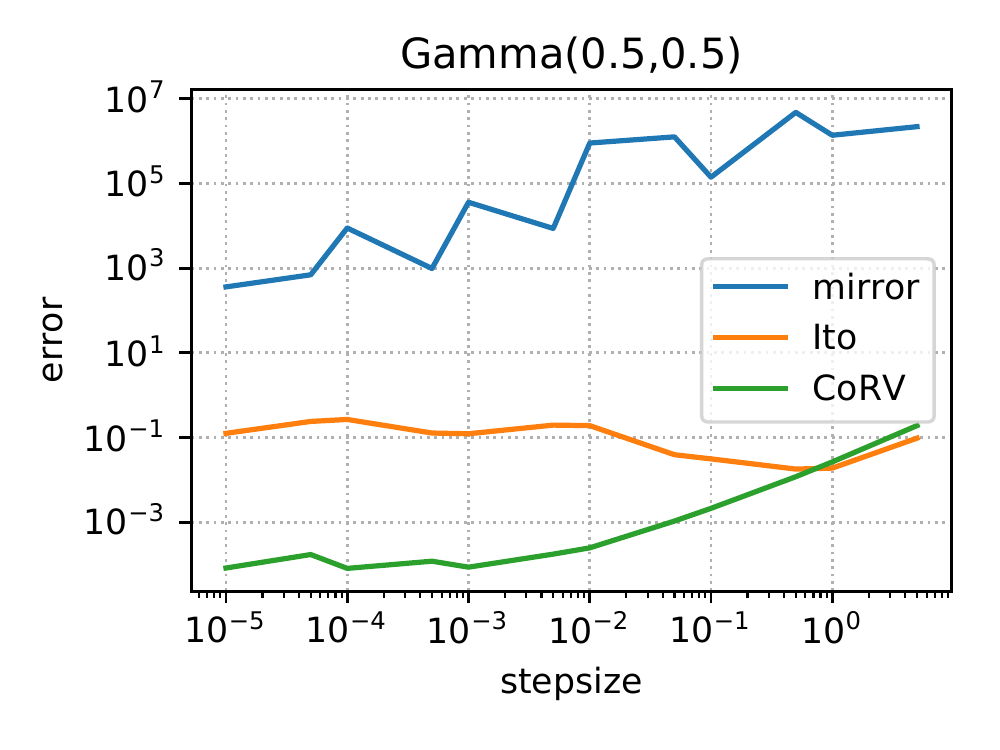}
  \includegraphics[width=0.32\linewidth]{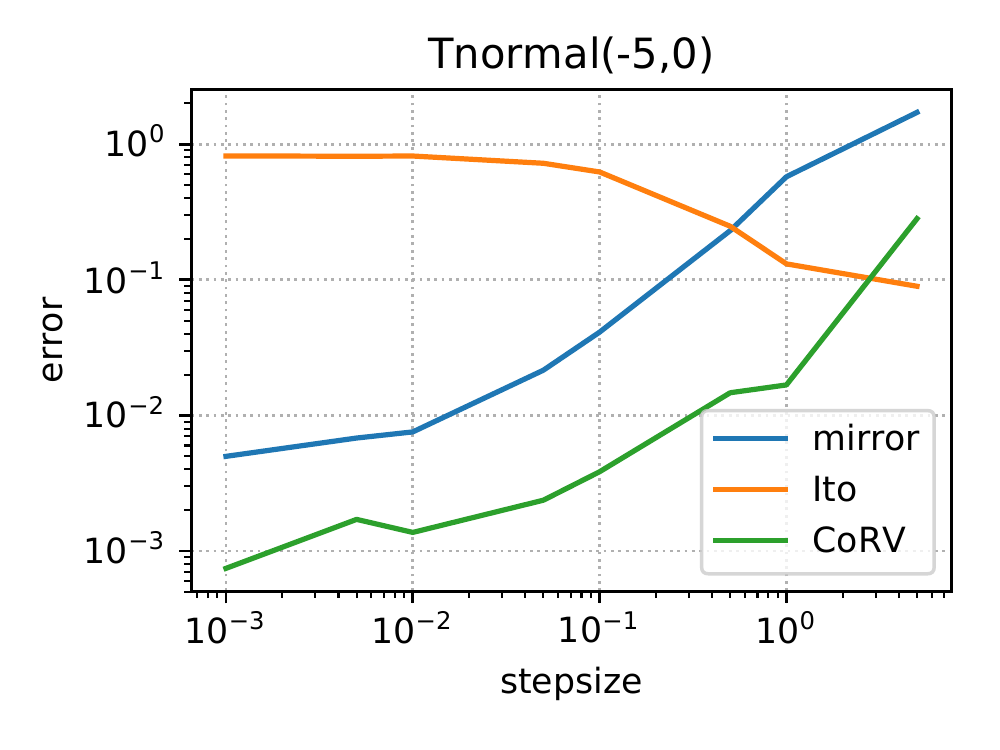}
  \caption{
    Expectation error in beta, gamma, and truncated standard normal distributions.
  }
  \label{fig:weakerror_compare}
\end{center}\end{figure*}

\section{Experiments} \label{sec:experiment}
In this section, we show the usefulness of our method using a range of models for many application scenarios.
Results demonstrate a practical efficacy of the CoRV approach on top of the theoretical justifications that have been discussed.
We used the P100 GPU accelerator for all experiments.

\begin{wraptable}{R}{0.45\textwidth}
    \caption{Transform functions.} 
  \label{tab:transform}
  \begin{center}
  \begin{tabular}{lc}
  \toprule
\multicolumn{1}{c}{Name}  &\multicolumn{1}{c}{Definition}\\
  \midrule
  sigmoid & $1/(1+\exp(-\varphi))\in (0,1)$\\
  arctan & $\tan^{-1}(\varphi)/\pi + 1/2\in (0,1)$\\
  softsign & $\varphi/2(1+|\varphi|) + 1/2\in (0,1)$\\
  \textcolor{red}{exp\footnotemark} & \textcolor{red}{$\exp(\varphi)\in \mathbb{R}_+$}\\
  softplus &  $\log(1+\exp(\varphi))\in \mathbb{R}_+$\\
  ICLL & $\varphi - \textrm{Ei}(-\exp(\varphi)) + \gamma\in \mathbb{R}_+$\\
  \bottomrule
  \end{tabular}
  \end{center}
\end{wraptable}

\subsection{Bayesian NMF}
For a typical application that uses a probability distribution supported on a finite or semi-infinite interval, we considered Bayesian non-negative matrix factorization~\citep{Cemgil:2009:BIN:1592511.1592515}.
We evaluated each sampling methods through the Bayesian prediction accuracy.
In the experiments, we employed the MovieLens dataset, a commonly used benchmark for matrix factorization tasks~\citep{Ahn:2015:LDB:2783258.2783373}.
It was split into $75:12.5:12.5$ for training, validation, and testing.
We compared (1) our CoRV, (2) the state-of-the-art SGLD-based method~\citep{Ahn:2015:LDB:2783258.2783373} modified for non-negative values, and (3) SGRLD~\citep{Patterson:2013kq} using natural gradient with diagonal preconditioning.
Methods (2) and (3) used the mirroring trick.
We also compared three transform functions for constraining to non-negative variables: exp, softplus, and ICLL in Table~\ref{tab:transform}.
The It\^o formulation was omitted due to a significant numerical instability.
The test root mean square error (RMSE) was used as the performance metric.
The prediction was given by the Bayesian predictive mean computed by a moving average.
We set the number of dimensions of latent variables $R$ to $20$ and $50$.
We trained for $10,000$ iterations with $R=20$ and for $20,000$ with $R=50$.
The stepsize was chosen by TPE of $100$ trials to minimize the validation loss.
Algorithm derivation and configuration are detailed in Appendix\,\ref{app:algo}.

\footnotetext{Note that the exponential function does not satisfy Assumption\,\ref{asm:transform_f}.}

\begin{figure*}[t]\begin{center}
    \includegraphics[width=0.49\linewidth]{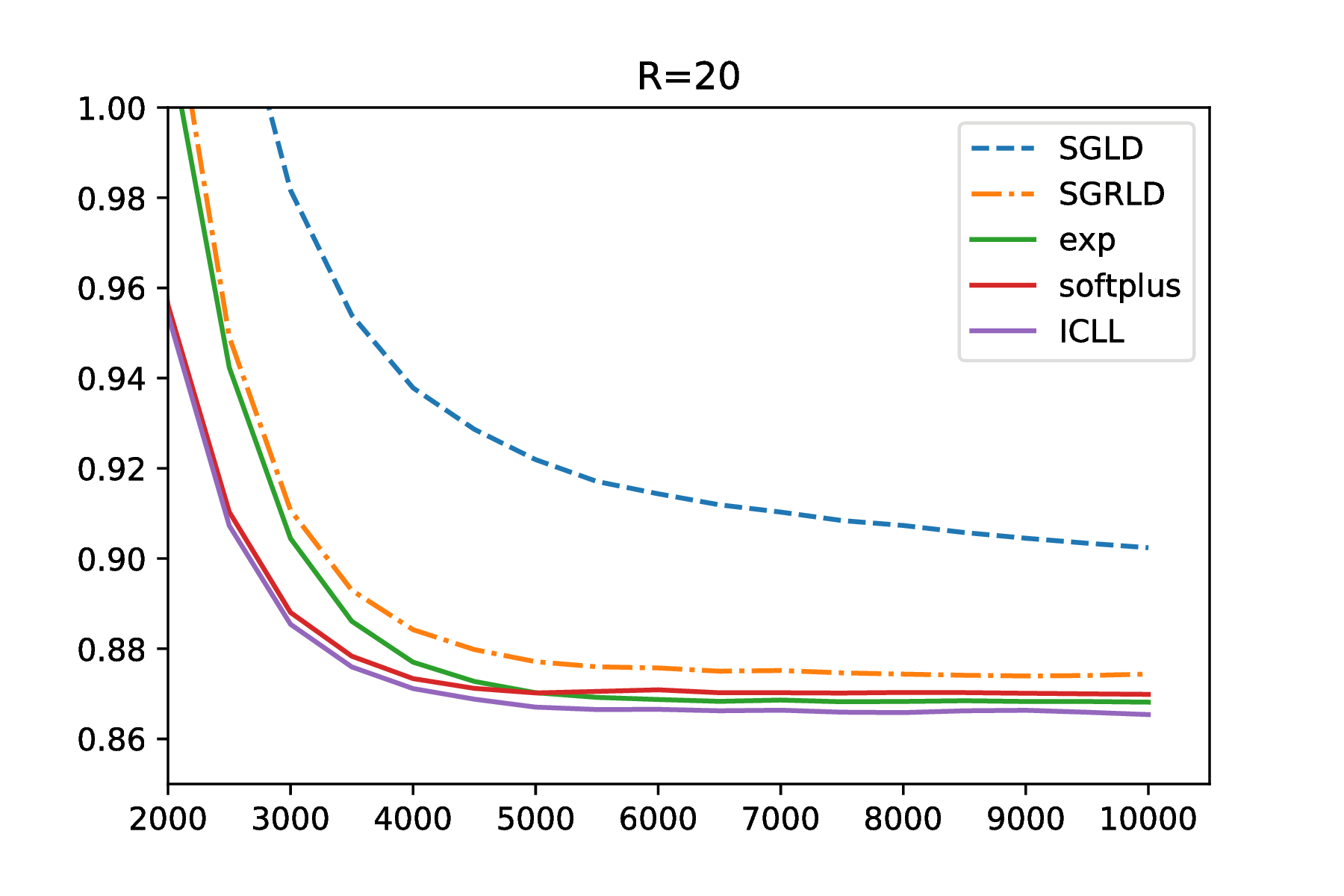}
    \includegraphics[width=0.49\linewidth]{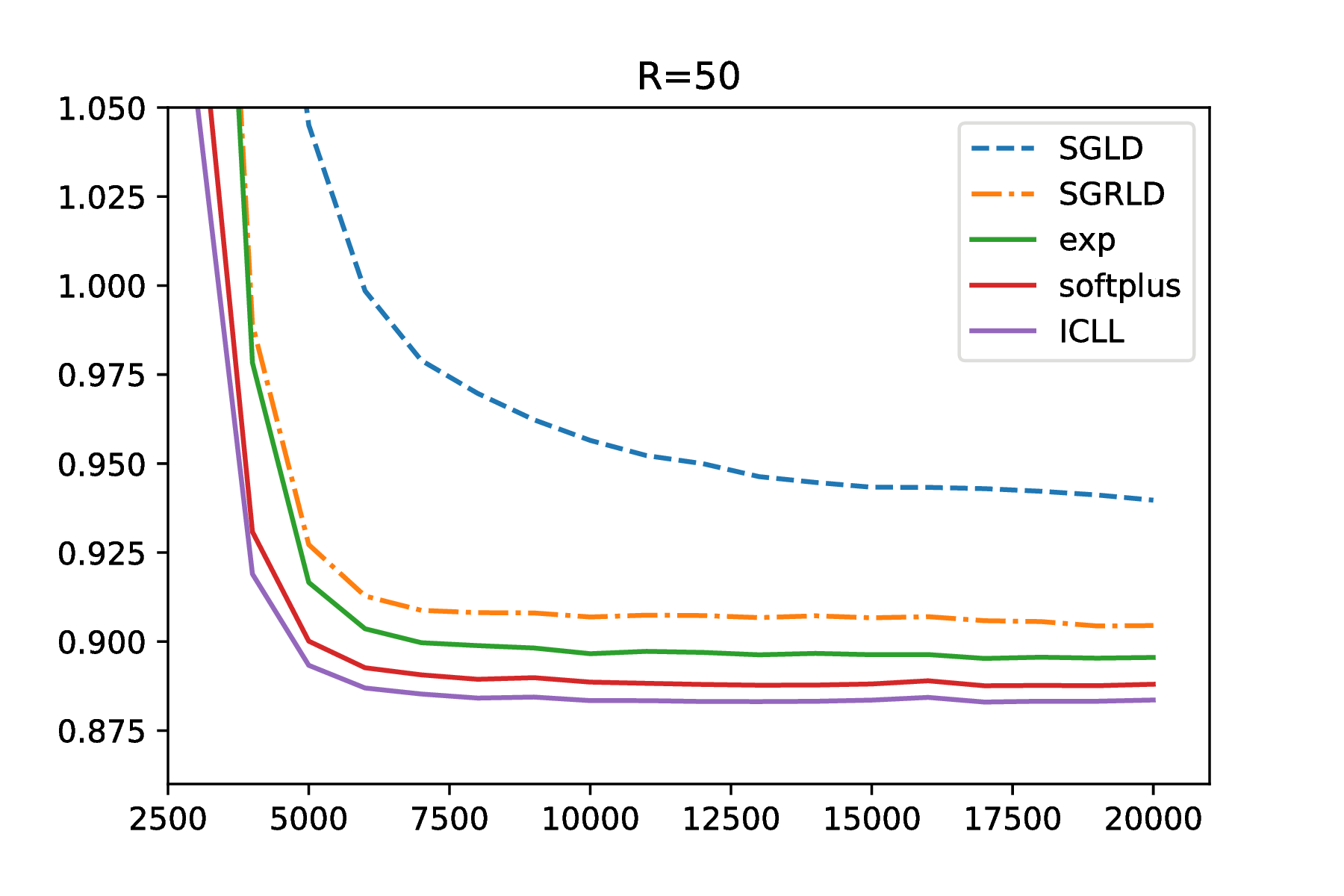}
    \caption{
        Test RMSE of Bayesian non-negative matrix factorization on the MovieLens dataset.
        Vertical and horizontal axes indicate RMSE and iteration number, respectively.
        Broken lines (SGLD and SGRLD) used the mirroring trick.
        Solid lines indicate CoRV method with respective transform functions.
        CoRV methods consistently outperformed existing methods with quick decrease in the error, while the choice of transform function slightly influenced the performance.
    }
    \label{fig:NMF_RMSE}
\end{center}\end{figure*}

\paragraph{Result.}
Figure~\ref{fig:NMF_RMSE} shows the curves of root mean square error (RMSE) values as a function of iterations.
SGLD and SGRLD are existing methods with the mirroring trick whereas exp, softplus, and ICLL indicates our method with the specified transform function.
We observed that CoRV SGLD made better predictions with smaller iterations than the other two algorithms.
When $R=20$ (Figure\,\ref{fig:NMF_RMSE} left), SGLD took $10,000$ iterations to reach an RMSE of $0.90$, whereas CoRV SGLD (softplus and ICLL) achieves it with only $3,000$ iterations.
While the choice of transform functions may influence the performance, CoRV outperformed the best performing baseline SGRLD.
Our method has a computational overhead regarding the transform function, as discussed in Appendix~\ref{app:compute}.
In this experiment, we found that at most $10\%$ computation time was necessary to run our method.

\begin{figure*}[t]\vspace{8mm}\begin{center}
    \includegraphics[width=0.49\linewidth]{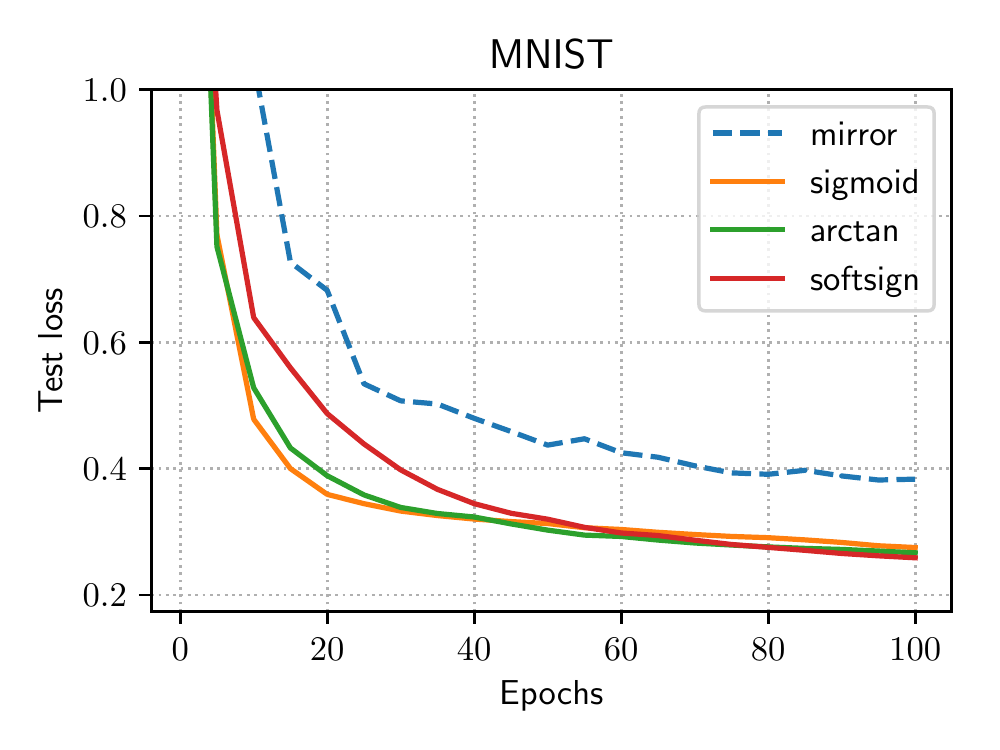}
    \includegraphics[width=0.49\linewidth]{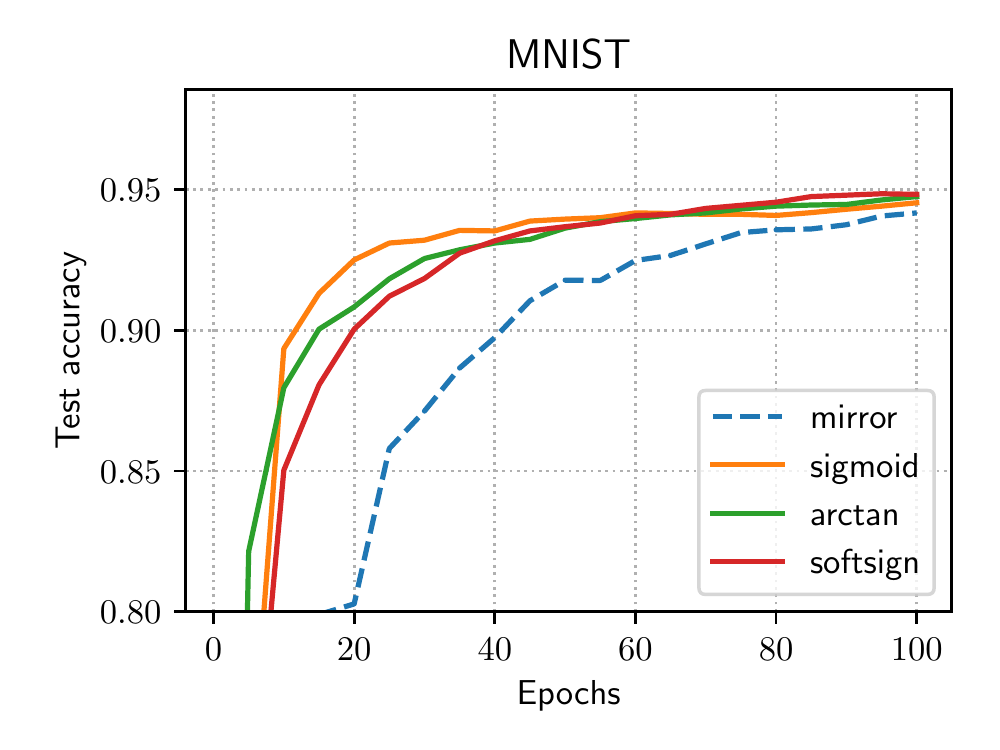}\hfill
    \includegraphics[width=0.49\linewidth]{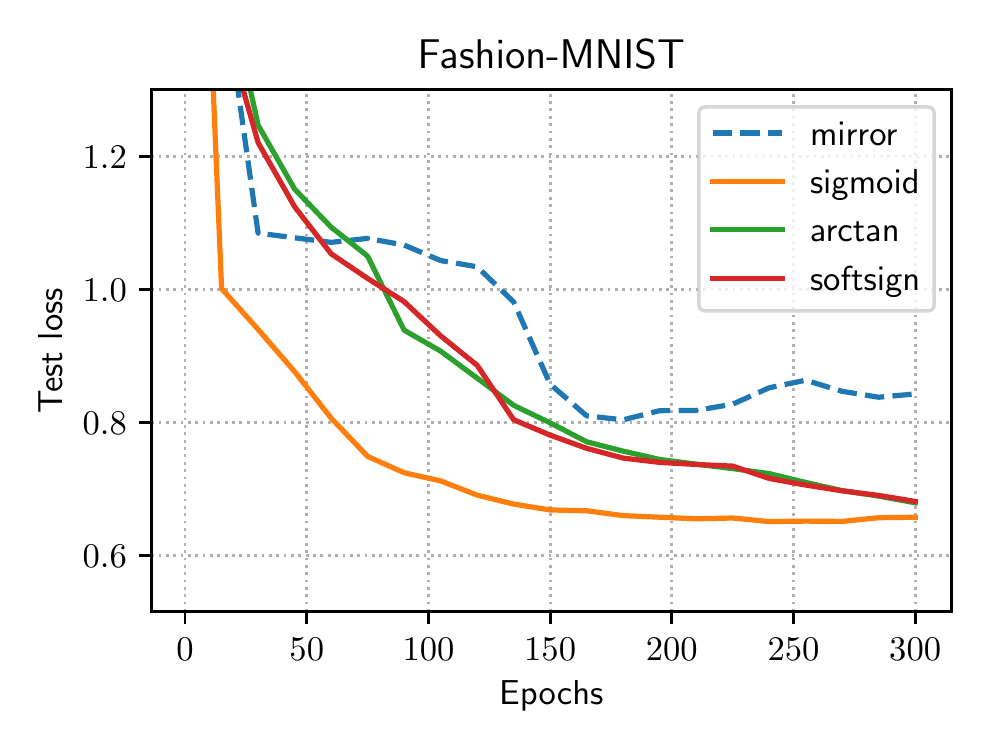}
    \includegraphics[width=0.49\linewidth]{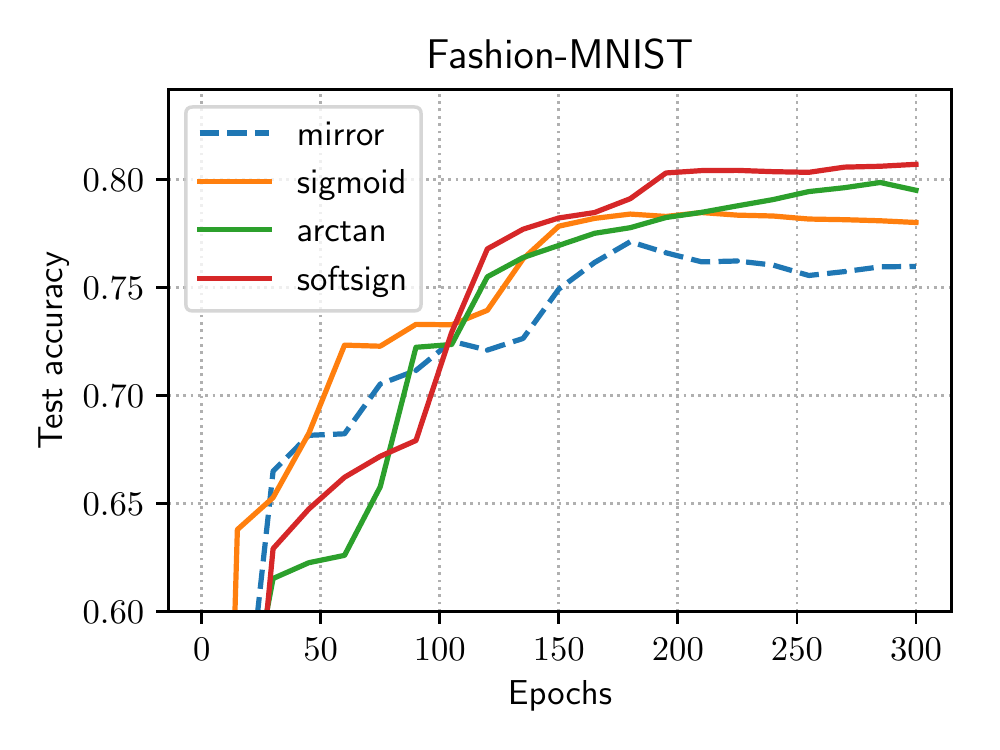}
    \caption{
        Test loss and accuracy of Bayesian binary neural networks on the MNIST and Fashion-MNIST dataset.
        The mirroring trick (\texttt{mirror}) showed a slower learning curve for both datasets.
        The CoRV formulation worked appropriately with the sigmoid, arctan, and softsign transforms.
    }
    \label{fig:BNN_result}
\end{center}\end{figure*}

\subsection{Bayesian binary neural network}
A binary neural network, whose parameters are restricted to binary, is expected to achieve high performance on small devices in terms of memory efficiency\,\citep{Courbariaux:2015aa}\,\citep{NIPS2016_6573}.
We evaluated each sampling methods through the Bayesian prediction accuracy of binary neural network model.
We considered a Bayesian binary three-layer feed forward network containing $50$ hidden units with the ReLU activation. 
In the experiments, we employed MNIST\,\citep{MNIST} and Fashion-MNIST\,\citep{fashion-MNIST} dataset.
Both datasets were split into $8:1:1$ for training, validation, and testing.
We compared
(1) our CoRV using the sigmoid, arctangent, and softsign functions, and
(2) a standard SGLD with the mirroring trick.
The It\^o formulation was omitted due to a significant numerical instability.
The cross-entropy loss of the softmax classifier and classification accuracy were evaluated.
The accuracy was given by the Bayesian predictive mean computed by a moving average of binarized weights at each epoch.
We trained the networks for $100$ epochs with MNIST and for $300$ epochs with Fashion-MNIST.
The stepsize was chosen by TPE of $100$ trials to minimize the validation loss.
The other experimental settings are in Appendix\,\ref{app:BNN_detail}.

\paragraph{Results.}
Figure\,\ref{fig:BNN_result} presents the test loss and accuracy.
Note that the purpose of this experiment is to compare sampling methods on the same model rather than to propose a state-of-the-art network.
The learning curves show that CoRV achieves better prediction than mirroring heuristics.
It is effective in practice that transformation enables stable computation with a large stepsize.

\section{Conclusion} \label{sec:conclusion}
SGLD has resorted to some heuristics for sampling bounded random variables since SGLD is designed for unbounded ones.
We demonstrated such heuristics may sacrifice the sampling accuracy both empirically and theoretically.
To deal with such random variables, we generalized SGLD using the change-of-random-variable (CoRV) formulation and analyzed its weak convergence.
Empirical evaluations showed that our CoRV SGLD outperformed existing heuristic alternatives on Bayesian non-negative matrix factorization and neural networks.

\section*{Acknowledgement}
Iseei Sato was supported by KAKENHI 17H04693.

\bibliography{ref}
\bibliographystyle{plainnat} %

\clearpage
\appendix

\section{It\^o Formula} \label{app:ito_formula}
In the stochastic differential equation, we have the following formula.
\begin{theorem}[It\^o Formula \citep{ito:1944}]\label{theorem:ito_formula}
$X(t)$ satisfies the stochastic differential equation
\begin{align}
dX(t)=a(t,X(t))dt+b(t,X(t))dW(t).
\end{align}
Let $h(t,X(t))$ be a given bounded function in $C^2((0,\infty)\times \mathbb{R})$.
Then, $h(t,X(t))$ satisfies the stochastic differential equation
\begin{align}
d h(t,X(t))=\mathcal{L}_1 h(t,X(t)) dt+ \mathcal{L}_2 h(t,X(t))dW(t),
\end{align}
where $\mathcal{L}_1$ and $\mathcal{L}_2$ are linear operators defined by
\begin{align}
\mathcal{L}_1 &=\frac{\partial}{\partial t}+a\frac{\partial}{\partial X} +\frac{1}{2}b^2\frac{\partial^2}{\partial X^2},~~~\mathcal{L}_2 =b\frac{\partial}{\partial X}.
\end{align}
\end{theorem}

\section{Proof}\label{app:proof}

\subsection{Proof of Lemma\,\ref{lem:proxy_potential}} \label{app:proof_proxy_potential}
\begin{proof}
    From the assumption of the target pdf, and $\varphi\to\infty$ as $\theta\to\partial S$,
    \begin{equation}
        \lim_{\varphi\to\infty} \pi(f(\varphi)) < C,
    \end{equation}
    for constant $C>0$.
    From Lemma\,\ref{lem:dfdx_limit},
    \begin{equation}
        \lim_{\varphi\to \infty} \pi(f(\varphi)) |f'(\varphi)| = 0.
    \end{equation}
    Using
    \begin{equation}
        U(\varphi) = - \log \left( \pi(\theta) |f'(\varphi)| \right),
    \end{equation}
    we have
    \begin{equation}
        \lim_{\varphi\to \infty} U(\varphi) = \infty.
    \end{equation}
\end{proof}

\subsection{Solution existence and uniqueness of SDE\,\eqref{equ:proxy_process}}\label{app:proxyprocess_exist_unique}
We check the existence of the solution of the SDE\,\eqref{equ:proxy_process}.
The following result is well-known.
\begin{lemma}[solution existence] \label{lem:existence_of_solution}
    Let $\widehat{U}'(\varphi)$ be a continuous function of $\varphi$.
    Then the solution of the SDE\,\eqref{equ:proxy_process} exists.
\end{lemma}

We also confirm the uniqueness of the solution.
We employ the weak uniqueness for the uniqueness in the sense of a distribution law.
\begin{theorem}[weak uniqueness\,\citep{Stroock:1979aa}] \label{thm:weak_uniqueness}
    Consider a $d$-dimensional SDE of $X\in\mathbb{R}^d$,
    \begin{equation} \label{equ:Stroock_SDE}
        dX(t) = a(X(t)) dt + b(X(t)) dW(t).
    \end{equation}
    Let $a(x)$ be a bounded measurable function for $x\in\mathbb{R}^d$.
    Let $B(x) = b(x)^\intercal b(x)$ be a bounded, continuous function where constant $K>0$ exists such that
    \begin{equation}
        \sum_{i,j=1}^{d} B_{ij}(x) \zeta_i \zeta_j \geq K |\zeta|^2,
    \end{equation}
    for $\zeta = (\zeta_1,\cdots,\zeta_d) \in \mathbb{R}^d$.
    Then the uniqueness in the sense of a distribution law holds for the solution of the SDE\,\eqref{equ:Stroock_SDE}.
\end{theorem}

The solution is unique under the following condition of $\widehat{U}'(\varphi)$.
\begin{lemma}[solution uniqueness] \label{lem:uniqueness_of_solution}
    Let the proxy potential gradient $\widehat{U}'(\varphi)$ be a bounded function.
    Then the solution of the SDE\,\eqref{equ:proxy_process} is unique in the sense of a distribution law.
\end{lemma}
\begin{proof}
    From Theorem\,\ref{thm:weak_uniqueness}, the condition of the diffusion coefficient is straightforwardly confirmed by letting $b = \sqrt{2}$ and $\zeta\in\mathbb{R}$, there exists constant $K>0$ such that
    \begin{equation}
        b^2 \zeta^2 \geq K \zeta^2.
    \end{equation}
\end{proof}

\subsection{Lemma~\ref{lem:dfdx_limit}} \label{app:dfdx_limit}
The following lemma is essential for showing the proxy potential and the solution of proxy SDE\,\eqref{equ:proxy_process}.
\begin{lemma}[limit of transform derivative] \label{lem:dfdx_limit}
    Under Assumption~\ref{asm:transform_f}, we have
    \begin{equation}
        \lim_{\varphi\to\infty} f'(\varphi) = 0.
    \end{equation}
\end{lemma}
\begin{proof}
    Using the L'H\^{o}pital's rule,
    \begin{equation}\begin{aligned}
        \lim_{\varphi\to\infty} f(\varphi) &= \lim_{\varphi\to\infty} \frac{\exp(\varphi) f(\varphi)}{\exp(\varphi)}\\
        &= \lim_{\varphi\to\infty} \frac{\exp(\varphi) (f(\varphi)+f'(\varphi))}{\exp(\varphi)}\\
        &= \lim_{\varphi\to\infty} (f(\varphi)+f'(\varphi)).
    \end{aligned}\end{equation}
    Thus we have
    \begin{equation}
        \lim_{\varphi\to\infty} f'(\varphi) = 0.
    \end{equation}
\end{proof}

\subsection{Lemma\,\ref{lem:proxy_gradient_error}} \label{app:proxy_gradient_error}
\begin{lemma}[proxy gradient error] \label{lem:proxy_gradient_error}
    Let $\delta$ be a noise of the stochastic gradient of the target potential that satisfies Assumption~\ref{asm:delta_SGLD}, and let $f$ satisfy Assumption~\ref{asm:transform_f}.
    Then for any noise $\delta_\varphi$ of the stochastic gradient of the proxy potential
    \begin{equation}
        \widehat{U}'(\varphi) = U'(\varphi) + \delta_\varphi,
    \end{equation}
    we have:
    \begin{equation}
        \mathbb{E}_{S}[\delta_\varphi] = 0, \ \ \ \ \ \mathbb{E}_{S}[|\delta_\varphi|^l] < \infty,
    \end{equation}
    for some integer $l \geq 2$.
\end{lemma}
\begin{proof}
    Since $\widehat{U}'_\theta(\theta)$ satisfies Assumption~\ref{asm:delta_SGLD}
    \begin{equation}
       \widehat{U}'_\theta(\theta) = U'_\theta(\theta) + \delta,
    \end{equation}
    as in Eq.\,\eqref{equ:convert_potential}, the stochastic gradient of the proxy potential is
    \begin{equation}\label{equ:proxy_noise}
      \begin{aligned}
        \widehat{U}'(\varphi)
            &= f'(\varphi) \left( U'_\theta(\theta) + \delta \right) - \frac{f''(\varphi)}{f'(\varphi)}\\
            &= U'(\varphi) + \delta_\varphi,
      \end{aligned}
    \end{equation}
    by letting $\delta_\varphi = f'(\varphi) \delta$.
    Since Assumption~\ref{asm:transform_f} suggests that the derivative of transform is always finite, $\delta_\varphi$ also satisfies zero mean and finite variance
    \begin{equation}
          \mathbb{E}_{S}[\delta_\varphi] = 0, \ \ \ \ \ \mathbb{E}_{S}[|\delta_\varphi|^l] < \infty.
    \end{equation}
\end{proof}

\subsection{Proof of Theorem\,\ref{thm:naive_unstable}} \label{app:proof_naive_unstable}
\begin{proof}
    From Lemma\,\ref{lem:dfdx_limit} in Appendix\,\ref{app:dfdx_limit},
    \begin{equation}
        \lim_{\varphi\to\infty} \frac{d}{d\varphi} g^{-1}(\varphi) = 0.
    \end{equation}
    This implies that
    \begin{equation}
        \lim_{\theta\to\partial S} |g'(\theta)| = \lim_{\varphi\to \infty} \frac{1}{\left|\frac{d}{d\varphi} g^{-1}(\varphi)\right|} = \infty.
    \end{equation}
    From Eq.\,\eqref{equ:naive_sampling}, the single-step difference is given by
    \begin{equation}
        |\varphi_{t+1} - \varphi_t| = \left| \epsilon_t\left(- g'(\theta_t)U'_\theta(\theta_t) + g''(\theta_t)\right) + \sqrt{2\epsilon_t} g'(\theta_t) \eta_t \right|.
    \end{equation}
    Considering $\eta_t\sim\mathcal{N}(0,1)$, the factor $g'(\theta )$ almost surely dominates this quantity. Therefore,
    \begin{equation}
        \lim_{\theta\to\partial S} |\varphi_{t+1} - \varphi_t| = \infty.
    \end{equation}
\end{proof}

\subsection{Proof of Theorem\,\ref{thm:CoRV_stable}} \label{app:proof_CoRV_stable}
\begin{proof}
    From Eq.\,\eqref{equ:proxy_noise} of Lemma\,\ref{lem:proxy_gradient_error} in Appendix\,\ref{app:proxy_gradient_error}, we have
    \begin{equation}
        \widehat{U}'(\varphi) = U'(\varphi) + \delta_\varphi,
    \end{equation}
    where $\delta_\varphi = f'(\varphi) \delta$ and $\delta$ satisfies Assumption~\ref{asm:delta_SGLD}.
    From Lemma~\ref{lem:dfdx_limit},
    \begin{equation}
        \lim_{\theta\to\partial S} \delta_\varphi = \lim_{\varphi\to\infty} f'(\varphi) \delta = 0.
    \end{equation}
\end{proof}

\subsection{Proof of Theorem\,\ref{thm:proxy_stationary}} \label{app:proof_proxy_stationary}
\begin{proof}
    From Lemma\,\ref{lem:existence_of_solution} and \ref{lem:uniqueness_of_solution}, there exists a unique solution in the sense of a distribution law.
  From Lemma\,\ref{lem:proxy_gradient_error}, the SDE\,\eqref{equ:proxy_process} satisfies the same assumption that \citet{Sato:2014xy} used for SGLD in unconstrained state space.
  The transition probability density function $p(\varphi,t)$ follows the Fokker-Planck equation
  \begin{equation}
    \frac{\partial}{\partial t} p(\varphi,t) = - \frac{\partial}{\partial \varphi} \left( U'_\varphi(\varphi) p(\varphi,t) \right) + \frac{\partial^2}{\partial \varphi^2} p(\varphi,t),
  \end{equation}
  and its stationary distribution is
  \begin{equation}
    \lim_{t\to\infty} p(\varphi,t) = \exp(-U(\varphi)) = \pi(\varphi).
  \end{equation}
  Note that $f'(\varphi(t))$ is always finite from Assumption~\ref{asm:transform_f}.
  Applying Eq.\,\eqref{equ:transform_dist}, we obtain the stationary distribution as
  \begin{equation}\begin{aligned}
    \lim_{t\to\infty} p(\theta,t)|f'(\varphi)| &= \pi_\theta(\theta)|f'(\varphi)|\\
    \therefore \lim_{t\to\infty} p(\theta,t) &= \pi_\theta(\theta).
  \end{aligned}\end{equation}
\end{proof}

\subsection{Proof of Theorem\,\ref{thm:proxy_weak}} \label{app:proof_proxy_weak}
\begin{proof}
    Let us consider stochastic differential equation
    \begin{equation}
    d\varphi(t) = a(\varphi(t))dt + b(\varphi(t))dW(t), \ \ \ 0 \leq t \leq T
    \end{equation}
    and its approximation in time $t_{k-1} \leq t \leq t_k$
    \begin{equation}
    d\widetilde{\varphi}(t) = \widetilde{a}(\varphi(t))dt + \widetilde{b}(\varphi(t))dW(t),
    \end{equation}
    where $\widetilde{a}(\varphi(t)) = a(\varphi(t)) + \delta_{\varphi,t}$.

    Using Lemma~\ref{lem:proxy_gradient_error} and Theorem 6 of \citet{Sato:2014xy}, for the test function $h$, we have
    \begin{equation}\begin{split}
    \left|\mathbb{E}[h(\widetilde{\varphi}(T))] - \mathbb{E}[h(\varphi(T))]\right|
    &= \Biggl| \int_0^T \mathbb{E}\left[ \left(\widetilde{a}(\varphi(t))-a(\varphi(t))\right) \frac{\partial}{\partial \varphi} \mathbb{E}[h(\widetilde{\varphi}(t))] \right] dt\\
    & \ \ + \int_0^T \frac{1}{2} \mathbb{E}\left[ \left(\widetilde{b}(\varphi(t))^2-b(\varphi(t))^2\right) \frac{\partial^2}{\partial \varphi^2} \mathbb{E}[h(\widetilde{\varphi}(t))] \right] dt \Biggr|
    \end{split}\end{equation}
    From the Weierstrass theorem, there exists constant $C_k>0$ such that
    \begin{equation}
    \mathbb{E}\left[ \left(\widetilde{a}(\varphi(t))-a(\varphi(t))\right) \frac{\partial}{\partial \varphi} \mathbb{E}[h(\widetilde{\varphi}(t))] \right] \leq C_k \epsilon_{t_{k-1}}
    \end{equation}
    \begin{equation}
    \mathbb{E}\left[ \left(\widetilde{b}(\varphi(t))^2-b(\varphi(t))^2\right) \frac{\partial^2}{\partial \varphi^2} \mathbb{E}[h(\widetilde{\varphi}(t))] \right] \leq C_k \epsilon_{t_{k-1}}
    \end{equation}
    for time $t_{k-1} \leq t \leq t_k$.
    Letting the maximum value of $C_k$ be $C_{\mathrm{max}}$ and $\epsilon_{t_{k-1}}$ be $\epsilon_0$,
    \begin{equation}
    \left|\mathbb{E}[h(\widetilde{\varphi}(T))] - \mathbb{E}[h(\varphi(T))]\right|
    < T C_{\mathrm{max}} \epsilon_0.
    \end{equation}
    That is, the sample of proxy variable $\varphi$ generated by Eq.\,\eqref{equ:propose_update} weakly converges
    \begin{equation}
    \left|\mathbb{E}[h(\widetilde{\varphi}(T))] - \mathbb{E}[h(\varphi(T))]\right|
    = \mathcal{O}(\epsilon_0).
    \end{equation}

    Let test function $h$ be a composition of transform function $f$ and test function $h_\theta$ in the target domain: $h(\cdot) = h_\theta(f(\cdot)))$.
    Thus, $h(\varphi(T)) = h_\theta(\theta(T))$ and $h(\widetilde{\varphi}(T)) = h_\theta(\widetilde{\theta}(T))$.
    The sample of target variable $\theta$ satisfies
    \begin{equation}\begin{aligned}
        |\mathbb{E}[h(\widetilde{\varphi}(T))] - \mathbb{E}[h(\varphi(T))]| &= \mathcal{O}(\epsilon_0)\\
        \therefore |\mathbb{E}[h_\theta(\widetilde{\theta}(T))] - \mathbb{E}[h_\theta(\theta(T))]| &= \mathcal{O}(\epsilon_0).
    \end{aligned}\end{equation}
\end{proof}

\section{Algorithm for Bayesian NMF\label{app:algo}}
Given the observed $I\times J$ matrix $X$, whose components take non-negative discrete values, we approximated it with a low-rank matrix product $W H$, where $W$ is $I\times R$ and $H$ is $R\times J$ non-negative matrix.
The prior distribution and likelihood are
\begin{equation}
  W_{ir} \sim \textrm{Exponential}(\lambda_W), \ \ \ \ \
    H_{rj} \sim \textrm{Exponential}(\lambda_H),
\end{equation}
\begin{equation}
  X_{ij} | W_{i:}, H_{:j} \sim \textrm{Poisson}\left(\sum_{r=1}^R W_{ir} H_{rj}\right),
\end{equation}
where $\lambda_W$ and $\lambda_H$ are hyper-parameters.

SGLD generated samples as follows using the stochastic gradient evaluated with a mini-batch:
\begin{equation} \label{equ:NMF_SGLD_W}
  W_{i:}^* = \left| W_{i:} - \epsilon_t \widehat{U}'_{W_{i:}} + \sqrt{2} \eta \right|,
\end{equation}
where noise $\eta$ conforms to $\mathcal{N}(0,I)$ with $I$ being the $R\times R$ identity matrix.
$W_{i:}^*$ denotes the sample at time $t+1$ given $W_{i:}$ is the sample at time $t$.
The absolute value is taken in an element-wise manner, which corresponds to the mirroring trick.
The stochastic gradient is
\begin{equation}
  \widehat{U}'_{W_{i:}} = - \frac{N}{|S|}\sum_{X_k\in S} H_{: j_k} \left( \frac{X_{k}}{\widehat{X}_{k}} -1 \right) + \lambda_W,
\end{equation}
where $j_k \in \{1,\cdots,J\}$ is the index of the $k$th data point in mini-batch $S$, $X_{k}$ is a discrete value of the $k$th data, and $\widehat{X}_{k} = \sum_{r=1}^R W_{i_k r} H_{r j_k}$ is its estimate.

CoRV SGLD updates proxy variables by
\begin{equation} \label{equ:NMF_propose_phiW}
  \varphi_{W_{i:}}^* = \varphi_{W_{i:}} - \epsilon_t \left( f'(\varphi_{W_{i:}}) \widehat{U}'_{W_{i:}} - \frac{f''(\varphi_{W_{i:}})}{f'(\varphi_{W_{i:}})} \right) + \sqrt{2} \eta.
\end{equation}
Here $f'$ and $f''$ are applied element-by-element.
The sample of $\varphi_H$ is obtained in the same manner.
Note that Eq.\,\eqref{equ:NMF_propose_phiW} bypasses the mirroring trick because proxy variables $\varphi_W$ and $\varphi_H$ are in the entire domain $\mathbb{R}$.
Matrices $W, H$ are always non-negative via transform $f: \mathbb{R} \to \mathbb{R}_+$.
The algorithms are shown below.

\begin{minipage}{.49\textwidth}
\begin{algorithm}[H]
\caption{SGLD for Bayesian NMF}
\begin{algorithmic}\label{alg:NMF_SGLD}
\STATE initialize $W^{(0)}, H^{(0)}$
\FOR{time $t \in 1,\cdots,T$}
\STATE subsample mini-batch $S_t$ from dataset
\STATE obtain new sample of $W^{(t)}$ by Eq.\,\eqref{equ:NMF_SGLD_W}
\STATE obtain new sample of $H^{(t)}$%
\ENDFOR
\STATE output $W^{(1)},\cdots,W^{(T)}$, $H^{(1)},\cdots, H^{(T)}$
\end{algorithmic}
\end{algorithm}
\end{minipage}\ \
\begin{minipage}{.49\textwidth}
\begin{algorithm}[H]
\caption{Transformed SGLD for Bayesian NMF}
\begin{algorithmic}\label{alg:NMF_propose}
\STATE initialize $W^{(0)}, H^{(0)}, \varphi_W^{(0)}, \varphi_H^{(0)}$
\FOR{time $t \in 1,\cdots,T$}
\STATE subsample mini-batch $S_t$ from dataset
\STATE obtain new sample of $\varphi_W^{(t)}$ by Eq.\,\eqref{equ:NMF_propose_phiW}
\STATE obtain new sample of $\varphi_H^{(t)}$
\STATE transform $W^{(t)} = f(\varphi_W^{(t)})$
\STATE transform $H^{(t)} = f(\varphi_H^{(t)})$
\ENDFOR
\STATE output $W^{(1)},\cdots,W^{(T)}$, $H^{(1)},\cdots,H^{(T)}$
\end{algorithmic}
\end{algorithm}
\end{minipage}

The data matrix consists of $I = 71,567$ users and $J = 10,681$ items with in total $10,000,054$ non-zero entries.
We set hyper-parameter $\lambda_W, \lambda_H$ to $1.0$, the number of dimensions of latent variables $R$ to $20$ and $50$, and the size of the mini-batch $|S|$ to $10,000$.

\section{Computational Complexity}\label{app:compute}
CoRV SGLD requires additional computation of the transformation step compared to the vanilla SGLD (see Algorithm~\ref{alg:NMF_propose} in Appendix~\ref{app:algo}).
In most cases, gradient computation is dominant in the SGLD calculation, which is proportional to the number of data in each mini-batch.
CoRV SGLD depends only on the number of parameters and does not change the complexity of gradient computation.
The influence on the computation time is limited, as the measured execution time was up to $+10\%$ at the maximum.

\section{Setting of Bayesian binary neural network}\label{app:BNN_detail}
The parameters were trained as continuous variables and binarized at prediction time to construct a Bayesian predictive distribution.
The weight parameter $w \in (-1,+1)$ had a prior of translated beta distribution, with hyper-parameter $\alpha,\beta$ and beta function $B(\alpha,\beta)$,
\begin{equation}
  p(w) = \frac{1}{2B(\alpha,\beta)} \left( \frac{1}{2}w+\frac{1}{2} \right)^{\alpha-1} \left( -\frac{1}{2}w+\frac{1}{2} \right)^{\beta-1}.
\end{equation}

\end{document}